\newif\ifarxiv
\arxivtrue

\ifarxiv
  \documentclass[11pt]{article}
  \usepackage{arxiv}
\else
  \documentclass[twoside]{article}
  \usepackage{aistats2027}
\fi
\usepackage{booktabs}
\usepackage{tabularx}
\usepackage{array}
\usepackage{amsmath,amssymb,amsthm}
\usepackage{graphicx}
\usepackage{microtype}
\usepackage{placeins}
\usepackage{float}
\ifarxiv
  \usepackage{url}
  \usepackage{axym-links}
  \publicationcode{https://github.com/Axym-Labs/kpsa}
  \publicationcorrespondence{dw@axym.org}
  \hypersetup{hidelinks}
  \title{Kernelized Parameter Sensitivity Atlases for Amortized Conditional Interpretability}
  \author{
  Davide Wiest$^{1}$ \quad Ahanaf Hasan Ariq$^{2}$ \quad Abdul Monaf Chowdhury$^{3}$\\[0.5em]
  {\small $^{1}$Technische Universit\"at Darmstadt \quad
  $^{2}$Daffodil International School \quad
  $^{3}$University of Dhaka}
  }
  \date{}
  \renewcommand{\shorttitle}{Kernelized Parameter Sensitivity Atlases for Amortized Conditional Interpretability}
\fi

\newtheorem{theorem}{Theorem}
\newtheorem{proposition}[theorem]{Proposition}
\newtheorem{corollary}[theorem]{Corollary}

\begin{document}

\ifarxiv
\maketitle
\else
\twocolumn[
\aistatstitle{Kernelized Parameter Sensitivity Atlases for Amortized Conditional Interpretability}
\aistatsauthor{Anonymous Authors}
\aistatsaddress{Anonymous Institution}
]
\fi

\begin{abstract}
Gradient-based attribution provides input-specific information about model parameters, but computing it via backpropagation for many samples is expensive, while static importance scores discard input dependence.
We derive KPSA (Kernelized Parameter Sensitivity Atlas) from a block decomposition of the empirical neural tangent kernel. It measures normalized parameter-group sensitivity on a fixed source set and functional, embeds those sources in an explicit query space, and uses a lightweight kernel to estimate sensitivities for unseen inputs.
We characterize query-space fit through target stability and derive bounds that separate finite-source from irreducible representation error and give margin-based top-$k$ recovery.
This paper studies two interpretability applications.
Firstly, KPSA identifies parameter groups with high susceptibility to perturbations across vision, forecasting, and protein-modeling setups. Across a fivefold range of parameter budgets ($0.02$--$0.1\%$), its selected groups attain $39$--$85\%$ of the perturbation susceptibility of a first-order backpropagation oracle. In vision and forecasting, these parameter groups form sparse sensitivity circuits that extend across depth and vary systematically with the input.
Secondly, we use KPSA on an activation-attribution problem for MLP features in transformers. It retains $81\%$ of the gap between a backpropagation oracle and a constant-sensitivity baseline on vision while increasing batched query-time throughput from 551 to 2,041 images/s and reducing peak GPU memory from 4.55 to 0.44~GB. On the protein-modeling task, it recovers $68\%$ of this gap.
An ablation of KPSA's kernel aggregation shows that nearest--neighbor retrieval in KPSA space can remain competitive.
The empirics show that compact kernel atlases can store useful input-dependent parameter sensitivity while amortizing repeated gradient attribution, a step toward scalable conditional interpretability.
\end{abstract}

\newcommand{\kpsaoverview}[1]{%
  \centering
  \includegraphics[width=#1]{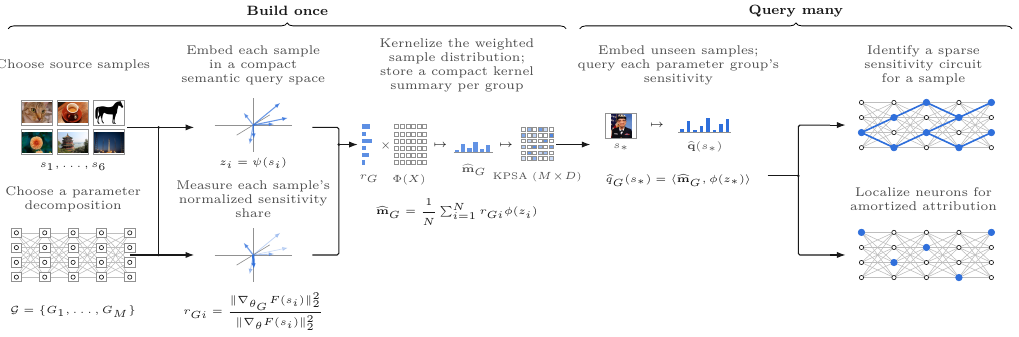}
  \caption{KPSA builds a compact atlas once by pairing source representations with normalized parameter-group sensitivities, then estimates conditional sensitivity profiles for sensitivity-circuit discovery and attribution without per-query backpropagation.}
  \label{fig:kpsa-overview}}

\ifarxiv
\makepublicationlinks

\begin{figure}[H]
  \kpsaoverview{0.85\linewidth}
\end{figure}
\else
\begin{figure*}[t]
  \kpsaoverview{\linewidth}
\end{figure*}
\fi

\section{Introduction}
\label{sec:introduction}

Gradient-based attribution estimates which parameters or internal features
affect a chosen model functional for a particular input.  A fresh backward
pass gives a high-resolution estimate, but its memory and latency costs recur
for every query.  A single global importance vector is cheaper, but removes
input dependence.  We ask whether this repeated-query problem can instead be
posed as statistical estimation: from a finite source set with measured
gradients, estimate the sensitivity profile of an unseen query using only a
compact representation and a forward computation.

Let a frozen scalar functional $F_\theta(x)$ induce an empirical neural
tangent kernel (NTK) over a partition of the parameters.  The NTK decomposes
exactly into parameter-block kernels, and their normalized diagonal
contributions form a sensitivity profile $s(x)$.  Given a query
representation $Z=\psi(X)$, the population object available without a query
gradient is generally the conditional mean
\begin{equation*}
 m_\psi(z)=\mathbb E[s(X)\mid Z=z],
\end{equation*}
rather than the realized profile $s(x)$.  More source gradients can reduce
estimation error for $m_\psi$, but cannot recover sensitivity variation that
$\psi$ discards.  A source-count-only guarantee to the realized query profile
therefore requires a representation sufficient for that profile.

We introduce KPSA, a Kernelized Parameter Sensitivity Atlas.  KPSA measures
the block-NTK profiles of a fixed source set, embeds the sources in a
categorical or learned query space, and transfers their profiles to unseen
inputs using a nonnegative retrieval kernel.  The experiments instantiate
this rule with exact source kernels, categorical delta kernels, or a compressed
kernel map.  Normalizing the retrieval weights does not change group rankings,
so the rule is a vector-valued Nadaraya--Watson estimator
\cite{nadaraya1964estimating,watson1964smooth} of an NTK-derived target.

The theory begins with the exact block decomposition of the empirical NTK.
Each diagonal block equals expected squared first-order susceptibility to a
random perturbation of that parameter group, so its normalized share defines
the quantity KPSA estimates.  We express query-space fit through the stability
of the conditional sensitivity target in the query metric.  A finite-source
bound separates irreducible representation error,
local target-instability error, and effective source count.  A margin
condition converts profile error into top-$k$ recovery, while a separate
smoothness result relates the first-order target to the finite perturbation
assay used in the experiments.

We evaluate the estimator in two interpretability uses and three modalities.
First, KPSA retrieves sparse sensitivity circuits whose exact antithetic
perturbation effects vary with the input in vision, forecasting, and protein
models.  At the $0.1\%$ budget, semantic KPSA recovers $35\%$, $47\%$, $50\%$,
and $33\%$ of the static-to-direct-gradient causal gap on ImageNet, ETTh1,
Weather, and Swiss-Prot.  Second, it localizes MLP features for exact
zero-deactivation, recovering $81\%$ of the vision causal gap at the efficiency
operating point and $68\%$ on protein modeling.  The vision query path runs at
2,041 rather than 551 images/s and uses 0.44 rather than 4.55~GB peak GPU
memory.  These gains follow from replacing per-query reverse-mode
differentiation with a forward representation and kernel query; construction
cost is paid once.

Our contributions are:
\begin{itemize}
  \item We derive KPSA from the block decomposition of the empirical NTK:
  normalized block diagonals are shares of expected infinitesimal parameter
  susceptibility and define the atlas target.
  \item We characterize query-space suitability through irreducible
  sensitivity risk and target stability, separating finite-source,
  local-smoothing, and representation errors.  The resulting bounds give
  top-$k$ recovery conditions and a smoothness bridge to finite perturbations.
  Complete proofs appear in Appendix~\ref{app:theory}.
  \item Without changing the frozen empirical protocol, we evaluate the same
  estimator across vision, forecasting, and protein modeling, measuring
  profile fidelity, exact perturbation and deactivation effects, sensitivity
  circuit structure, storage, memory, and throughput.
\end{itemize}

\section{Related Work}
\label{sec:related-work}

\paragraph{Nonparametric regression and kernel approximation.}
The KPSA query rule is local-constant kernel regression
\cite{nadaraya1964estimating,watson1964smooth}; its bias--variance rate depends
on response smoothness and local covariate dimension
\cite{tsybakov2009introduction}.  Kernel mean embeddings represent
distributions as RKHS elements \cite{smola2007hilbert,muandet2017kernel}, and
conditional mean embeddings admit a vector-valued regression interpretation
\cite{grunewalder2012conditional}.  Explicit and compressed kernel maps reduce
storage and query cost
\cite{pham2013fast,oglic2017nystrom,cai2022fast,chatalic2022nystrom}.  KPSA
uses these tools for an NTK-derived simplex-valued response whose coordinates
are native parameter groups.  The analysis specializes standard
local-regression arguments to query-space target stability and top-$k$
parameter selection; it does not propose a new generic regression rate.

\paragraph{Gradient geometry and compact sensitivity summaries.}
The NTK is an inner product of parameter-gradient features
\cite{jacot2018neural}; Fisher kernels and empirical Fisher summaries also
encode gradient geometry
\cite{jaakkola1999exploiting,wu2024improved}, and low-rank neural Fisher
kernels show that this geometry can admit compact approximations
\cite{zhang2022neural}.  Task2Vec and related task embeddings aggregate Fisher
information into task-level representations
\cite{achille2019task2vec,achille2021information,liu2025wasserstein}.  The
empirical Fisher used by EWC and the improved empirical Fisher proposed by
IEWC are closely related to the quantities stored in a sensitivity atlas.
KPSA can therefore, in principle, be used for input-conditioned,
regularization-based continual learning
\cite{kirkpatrick2017overcoming,wiest2026improved}.

FisherSketch is the closest kernelized precedent.  For each task, it averages
random features of joint activation--error pairs into a single empirical mean
vector in product-kernel feature space; cosines between these vectors estimate
task-level Fisher alignment \cite{sweeney2026geometry}.  KPSA preserves a
different object, the block-diagonal allocation of one functional's gradient
energy, together with its association to source representations, and estimates
the allocation conditionally at each query.  A single task-level mean vector
may retain multimodal distributional information, but it is not indexed by a
query.  The corresponding unconditional profile mean is KPSA's constant
sensitivity baseline; local kernel estimation instead lets separated
sensitivity regions contribute differently to each query.  Both constructions
compress Fisher-derived information, but their storage is not directly
comparable because FisherSketch stores one vector per task whereas KPSA stores
information for every actionable parameter group.  KPSA permits categorical
as well as learned query spaces; its random-projection control tests whether
generic compression alone supplies useful neighborhoods.  This conditioning
is pertinent when neurons are polysemantic or participate in task-specific
circuits \cite{foote2024tackling,song2024large,leng2025towards,marks2025sparse}.

\paragraph{Influence and amortized attribution.}
Influence functions and gradient-tracing methods estimate training-example
effects \cite{koh2017understanding,pruthi2020estimating}; gradient embeddings
also support active selection \cite{ash2020deep}.  Mechanistic data attribution
and gradient routing connect examples or tasks to internal computation
\cite{chen2026mechanistic,cloud2024gradient}.  Learned amortizers such as
FastSHAP replace repeated attribution optimization with a predictor
\cite{jethani2022fastshap}.  KPSA is training-free after source profiling: it
stores block-NTK sensitivity profiles and applies a fixed kernel query rule
with an explicit target and finite-source error decomposition.

\paragraph{Neuron and circuit localization.}
Prior work localizes task- or fact-relevant neurons and tests them through
intervention
\cite{dai2022knowledge,song2024large,leng2025towards,foote2024tackling,
park2022neurocartography}.  HONES performs gradient-free, task-aware neuron
attribution and steering in vision--language models \cite{wang2026heads}.
Automated circuit discovery and attribution patching search over model
components \cite{conmy2023automated,syed2024attribution}, while recent
benchmarks and sparse-circuit work emphasize causal faithfulness
\cite{mueller2025mib,marks2025sparse,arora2026language}.  Our use of
\emph{sensitivity circuit} is narrower: a conditional set of native parameter
groups with measured local causal utility, without a claim of signed directed
edges.  KPSA amortizes this localization and characterizes when the selected
groups match a first-order reference.

\section{Kernelized Parameter Sensitivity Atlases}
\label{sec:theory}

\subsection{From block NTKs to parameter sensitivity}

Fix a differentiable scalar functional $F_\theta:\mathcal X\to\mathbb R$ of
a frozen model and a disjoint partition
$\mathcal P=\{G_1,\ldots,G_M\}$ of the profiled parameters.  Write
$g_m(x)=\nabla_{\theta_{G_m}}F_\theta(x)$ and define the blockwise empirical
neural tangent kernels
\begin{equation}
 \begin{aligned}
 K_m(x,x')&=\langle g_m(x),g_m(x')\rangle,\\
 K_{\mathcal P}(x,x')&=\sum_{m=1}^M K_m(x,x').
 \end{aligned}
 \label{eq:block-ntk}
\end{equation}
The second equality follows exactly from the disjoint parameter partition.
On the diagonal, $K_m(x,x)=\|g_m(x)\|_2^2$ and
$K_{\mathcal P}(x,x)=\sum_m\|g_m(x)\|_2^2$.  KPSA normalizes this diagonal
decomposition:
\begin{equation}
 s_m(x)=\frac{K_m(x,x)}{K_{\mathcal P}(x,x)},
 \qquad s(x)=(s_1(x),\ldots,s_M(x)).
 \label{eq:kpsa-profile}
\end{equation}
Whenever $K_{\mathcal P}(x,x)>0$, $s(x)$ lies in the probability simplex
$\Delta_M=\{u\in\mathbb R_+^M:\|u\|_1=1\}$.  If an application stores only
an actionable subset of the profiled groups, the retained coordinates form a
sub-probability vector.  The normalization preserves where sensitivity lies
while discarding the input-dependent total gradient scale.

\begin{proposition}[Block-NTK sensitivity]
\label{prop:block-ntk-sensitivity}
Let $\xi_m$ have independent Rademacher coordinates on $G_m$ and zeros
elsewhere.  Then
\begin{equation}
 \mathbb E_{\xi_m}\left[(g_m(x)^\top\xi_m)^2\right]=K_m(x,x).
 \label{eq:block-susceptibility}
\end{equation}
Thus $s_m(x)$ is block $m$'s share of expected infinitesimal perturbation
susceptibility.  For equal-mass groups, the top-$k$ coordinates of $s(x)$
maximize captured first-order susceptibility among all $k$-group selections.
\end{proposition}

The sensitivity target is therefore the normalized diagonal allocation of the
empirical NTK across native parameter groups, fixed before choosing a query
representation or retrieval kernel.

\subsection{Retrieval in a query space}

Let $Z=\psi(X)$ be a categorical or learned query representation.  The profile
predictable without a query backward pass is
\begin{equation}
 m_\psi(z)=\mathbb E[s(X)\mid Z=z].
 \label{eq:kpsa-target}
\end{equation}
It equals the realized profile $s(x)$ only when the representation retains all
profile-relevant variation.  Define the query space's irreducible sensitivity
risk by
\begin{equation}
 \mathcal I_\psi
 =\mathbb E\!\left[\|s(X)-m_\psi(\psi(X))\|_2^2\right].
 \label{eq:irreducible-risk}
\end{equation}

\begin{proposition}[Query-space sensitivity sufficiency]
\label{prop:query-sufficiency}
For every square-integrable predictor $f$ from the query space to
$\mathbb R^M$,
\begin{equation}
 \mathbb E\|s(X)-f(Z)\|_2^2
 =\mathcal I_\psi+\mathbb E\|m_\psi(Z)-f(Z)\|_2^2.
 \label{eq:query-sufficiency}
\end{equation}
Consequently, no predictor using only $Z$ can have smaller mean-squared
profile error than $\mathcal I_\psi$, and $\mathcal I_\psi=0$ exactly when the
block-NTK sensitivity profile is determined by the representation almost
surely.
\end{proposition}

Equation~\ref{eq:query-sufficiency} separates the query-space choice from the
retrieval rule.  A representation can merge inputs with different sensitivity
profiles even with infinitely many sources.  Among representations with small
$\mathcal I_\psi$, the remaining question is whether the conditional target
varies smoothly enough for local retrieval.

Given source pairs $(z_i,s_i)_{i=1}^N$, where $z_i=\psi(x_i)$ and
$s_i=s(x_i)$, let $\kappa_h$ be a nonnegative retrieval kernel.  Define
\begin{equation}
 \lambda_i(z)=\frac{\kappa_h(z_i,z)}{\sum_{j=1}^N\kappa_h(z_j,z)},
 \qquad
 \widehat m(z)=\sum_{i=1}^N\lambda_i(z)s_i.
 \label{eq:kpsa-estimator}
\end{equation}
KPSA ranks the corresponding unnormalized scores
$\sum_i\kappa_h(z_i,z)s_i$; the common positive denominator in
Equation~\ref{eq:kpsa-estimator} does not change their ordering.  The retrieval
kernel $\kappa_h$ is distinct from the model NTK $K_{\mathcal P}$: the latter
defines the target, while the former transfers that target between inputs.
The experiments use exact source-kernel evaluation, categorical delta kernels,
and finite prototype-response maps.  Nystr\"om is another compatible
compression choice rather than part of the KPSA definition.

The local fit of a query space depends on two corresponding quantities.  Its
pointwise irreducible residual is
\begin{equation*}
 \varepsilon_\psi(x)=\|s(x)-m_\psi(\psi(x))\|_\infty,
\end{equation*}
and target stability is expressed by the H\"older condition
\begin{equation}
 \|m_\psi(z)-m_\psi(z')\|_2\leq Ld(z,z')^\alpha,
 \qquad 0<\alpha\leq1.
 \label{eq:target-stability}
\end{equation}
A suitable representation places inputs with similar block-NTK profiles near
one another, reducing both local variation and information lost by
conditioning on $Z$.

\paragraph{Assumptions.}
(A1) $F_\theta$, $\mathcal P$, $\psi$, and $\kappa_h$ are fixed, and the source
pairs are independent draws from one population.  (A2)
$K_{\mathcal P}(X,X)>0$ almost surely and $0\leq\kappa_h\leq1$.  (A3)
Equation~\ref{eq:target-stability} holds in the query metric.  The i.i.d. rate
below is not invoked for the fixed, overlapping forecasting windows; the
deterministic ranking statements still apply there.

\subsection{Finite-source error and query-space fit}

For a query $z$, define the effective source count and weighted radius by
\begin{equation}
 n_{\mathrm{eff}}(z)=\frac{1}{\sum_i\lambda_i(z)^2},
 \qquad
 \rho_\alpha(z)=\sum_i\lambda_i(z)d(z_i,z)^\alpha.
 \label{eq:neff-radius}
\end{equation}
Both quantities are observable from the source representations and query.

\begin{theorem}[Finite-source sensitivity gap]
\label{thm:truth-gap}
Under A1--A3, fix $x_*$ with $z_*=\psi(x_*)$ and nonzero retrieval-kernel
mass.  Conditional on the source representations, with probability at least
$1-\delta$ over their profiles,
\begin{equation}
 \|\widehat m(z_*)-s(x_*)\|_\infty
 \leq \varepsilon_\psi(x_*)+L\rho_\alpha(z_*)
 +\frac{1+\sqrt{\log(1/\delta)}}{\sqrt{n_{\mathrm{eff}}(z_*)}}.
 \label{eq:truth-gap}
\end{equation}
\end{theorem}

The three terms in Equation~\ref{eq:truth-gap} are irreducible representation
error, local target-instability error, and finite-source error.  More sources
reduce only the last term.  A smaller retrieval radius reduces the stability
term only while enough local sources remain, and no source count can remove
sensitivity variation discarded by $\psi$.  The stochastic term is independent
of the number of parameter groups because every profile lies in a simplex.

\begin{corollary}[Local consistency and rate]
\label{cor:local-rate}
Suppose the query metric has bounded diameter and the representation
distribution has local small-ball dimension $d$ near $z_*$.  For an RBF
retrieval kernel with $h_N\to0$ and $Nh_N^d\to\infty$,
\begin{equation}
 \|\widehat m(z_*)-m_\psi(z_*)\|_\infty
 =O_p\!\left(h_N^\alpha+(Nh_N^d)^{-1/2}\right).
 \label{eq:local-rate}
\end{equation}
Choosing $h_N\asymp N^{-1/(2\alpha+d)}$ gives the standard local-regression
profile rate $O_p(N^{-\alpha/(2\alpha+d)})$
\cite{tsybakov2009introduction}.  For the fixed bandwidths used in our
experiments, the sampling term shrinks toward a nonzero smoothing-bias floor.
\end{corollary}

\subsection{Normalization and selection consequences}

The block-NTK normalization is exactly invariant to a common positive scaling
of all parameter gradients.  More generally, if nonnegative block energies
are perturbed from $e$ to $\widetilde e=ae+\eta$ with $a>0$ and
$\|\eta\|_1\leq\gamma a\|e\|_1$, $\gamma<1$, then
\begin{equation}
 \left\|\frac{\widetilde e}{\|\widetilde e\|_1}
 -\frac{e}{\|e\|_1}\right\|_1
 \leq\frac{2\gamma}{1-\gamma}.
 \label{eq:normalization-stability}
\end{equation}
Thus small non-common energy error remains controlled, although normalization
cannot repair an unsuitable query representation.

Let $S_k(v)$ denote the indices of the $k$ largest coordinates of $v$, and let
$\Delta_k(v)=v_{(k)}-v_{(k+1)}$.

\begin{proposition}[Selection recovery and first-order regret]
\label{prop:selection}
If $\|\widehat v-v\|_\infty\leq b$, then
\begin{equation}
 \sum_{m\in S_k(v)}v_m-
 \sum_{m\in S_k(\widehat v)}v_m\leq2kb.
 \label{eq:selection-regret}
\end{equation}
If additionally $\Delta_k(v)>2b$, then
$S_k(\widehat v)=S_k(v)$.  Combining this proposition with
Theorem~\ref{thm:truth-gap} gives exact recovery of the query-gradient top-$k$
groups whenever the sensitivity bound is smaller than half the reference
margin.
\end{proposition}

By Proposition~\ref{prop:block-ntk-sensitivity},
Equation~\ref{eq:selection-regret} controls lost infinitesimal perturbation
susceptibility.  The finite-step central-difference assay in
Section~\ref{sec:application2} also contains nonlinear effects;
Appendix~\ref{app:theory} gives a smooth local bridge rather than identifying
the two quantities.

\paragraph{Complexity.}
Atlas construction requires $N$ source backward passes.  Exact retrieval
stores the $N\times M$ profile table and costs $O(NM)$ per query.  With a
$D$-dimensional explicit or compressed feature map, construction costs
$O(NMD)$ arithmetic and $O(MD)$ storage; a query forms the $D$ features,
combines them with the stored coefficients in $O(MD)$ time, and performs
partial or full ranking, with no query backward pass.  Held-out profile
retrieval, $n_{\mathrm{eff}}$, and matched-shuffle controls test whether the
chosen query space supplies stable local neighborhoods.

\section{Sensitivity Circuits}
\label{sec:application2}

\subsection{Setup}

We test a frozen AugReg ViT-B/16 classifier
\cite{dosovitskiy2021image,steiner2022train} on ImageNet
\cite{deng2009imagenet}, using DINOv2-S \cite{oquab2023dinov2} for its learned
query space; Chronos-Bolt-small \cite{ansari2024chronos} on the ETTh1 and
Weather series \cite{zhou2021informer,wu2021autoformer}; and ESM2-35M
\cite{lin2023evolutionary} on reviewed Swiss-Prot sequences
\cite{uniprot2025}. Appendix~\ref{app:protocol} gives the source/query splits,
functionals, group decompositions, kernels, interventions, and uncertainty
estimates.

We first evaluate whether one atlas can identify sparse parameter-group sets
whose exact perturbation effects remain input dependent. The same budgets,
method family, causal assay, and uncertainty estimator are used across the
three tested modalities in Figure~\ref{fig:application2-primary}.

\begin{figure}[t]
  \centering
  \ifarxiv\includegraphics{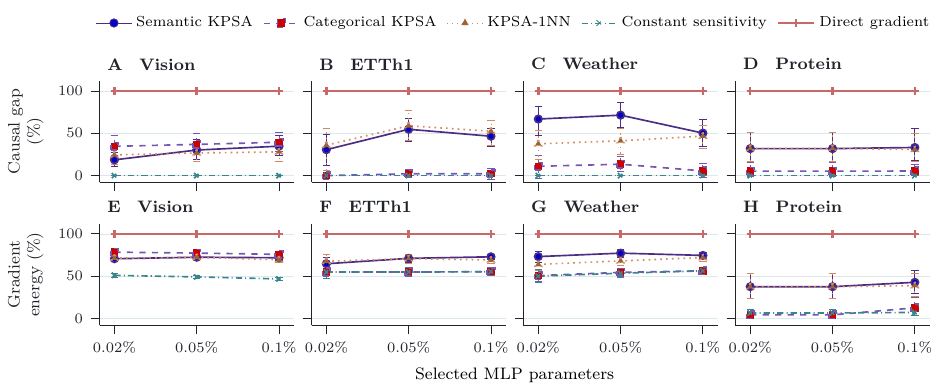}\else\includegraphics{figures/rendered/application2_primary.pdf}\fi
  \caption{\textbf{Conditional parameter sensitivity transfers across vision, forecasting, and protein models.} \ifarxiv Rows\else Columns\fi{} report the constant-to-direct-gradient causal-susceptibility gap recovered and direct-gradient energy captured; \ifarxiv columns\else rows\fi{} are ImageNet-1k/ViT-B/16, ETTh1 and Weather/Chronos-Bolt, and masked-residue prediction/ESM2. For legibility, we report susceptibility normalized to the constant-sensitivity baseline (0\%) and the first-order backpropagation oracle (100\%), because raw susceptibility scales are incompatible across settings. All panels compare semantic and categorical KPSA, the KPSA-1NN ablation, constant sensitivity, and the direct per-query gradient at 0.02\%, 0.05\%, and 0.1\% budgets. Interventions are exact antithetic parameter perturbations; points and bars are query means and paired 95\% bootstrap intervals.}
  \label{fig:application2-primary}
\end{figure}

At the $0.1\%$ budget, semantic KPSA recovers $34.7\%$, $46.5\%$,
$50.2\%$, and $33.2\%$ of the constant-to-direct causal gap on vision,
ETTh1, Weather, and protein modeling, while its selected sets contain
$71.8\%$, $72.9\%$, $74.6\%$, and $42.8\%$ of the same-budget direct
gradient energy (Table~\ref{tab:application2-operating-point}).  Thus the
method estimates a first-order target while evaluation remains an exact
forward perturbation assay; Proposition~\ref{prop:local-perturbation} gives
the smooth local bridge between them and makes clear that finite perturbations
can retain nonlinear error.  The
categorical query space is competitive in vision but contributes little on
the two forecasting transfers and protein task. The KPSA-1NN ablation is
slightly stronger on ETTh1 and slightly weaker on Weather and protein,
showing that local averaging is not uniformly preferable; its advantage
depends on whether the query space supplies multiple relevant neighbors.

Figure~\ref{fig:application2-circuits} tests the stronger circuit claim through
depth, representation-conditioned overlap, and exact causal controls matched
for both budget and layer allocation.

\ifarxiv\begin{figure}[!htb]\else\begin{figure}[t]\fi
  \centering
  \ifarxiv\includegraphics{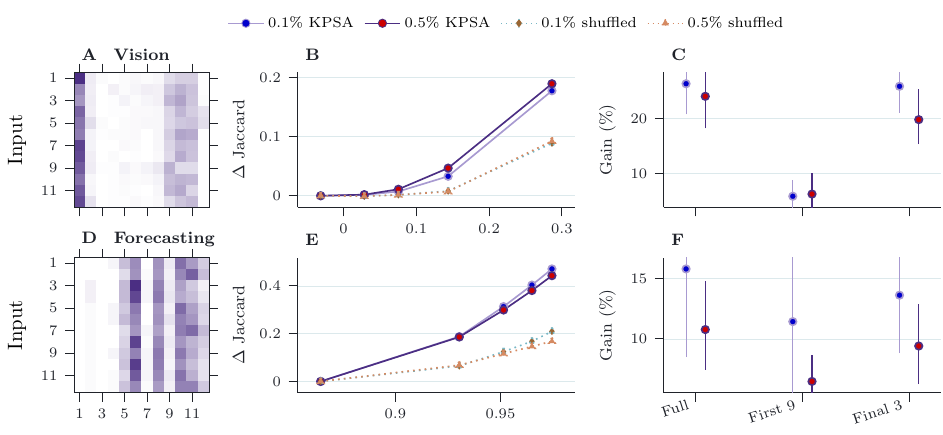}\else\includegraphics{figures/rendered/application2_circuits.pdf}\fi
  \caption{\textbf{Semantic KPSA identifies causally useful sensitivity circuits across vision and forecasting.} \ifarxiv Rows\else Columns\fi{} use the same protocol on ImageNet-1k/ViT-B/16 and ETTh1/Chronos-Bolt. Panels \ifarxiv A and D\else A--B\fi{} show circuit mass across transformer depth for 12 held-out inputs. \ifarxiv B and E\else C--D\fi{} relate representation similarity to circuit overlap against shuffled source--profile pairing. \ifarxiv C and F\else E--F\fi{} report exact perturbation susceptibility above a constant selection with matched per-layer counts, normalized by direct-gradient susceptibility. Full-stack, first-nine-block, and final-quarter selections use 0.1\% and 0.5\% budgets. \emph{Sensitivity circuit} here denotes a conditional set of native parameter groups with exact local causal utility, without asserting signed directed edges.}
  \label{fig:application2-circuits}
\end{figure}

At $0.1\%$, KPSA selections occupy an average of 6.0 ViT blocks and 5.8
Chronos blocks.  From the lowest-similarity bin to the top $1\%$, their
Jaccard coefficient increases by 0.177 in vision and 0.471 in forecasting,
compared with 0.089 and 0.210 after shuffling the representation--profile
pairing.  Absolute top-bin overlap is 0.389 and 0.614.  The exact
susceptibility gain over a layer-count-matched constant selection is
$26.3\%$ of the direct-gradient scale in vision and $15.8\%$ in forecasting.
Restricting selection to the final quarter retains most of this gain, whereas
the first nine blocks retain substantially less, especially in vision.  These
tests support conditional sparse sensitivity circuits but not a signed-edge circuit
interpretation.

\FloatBarrier

\section{Amortized Neuron Localization}
\label{sec:application1}

\subsection{Setup}

We evaluate the same frozen AugReg ViT-B/16 and DINOv2-S representation
\cite{dosovitskiy2021image,steiner2022train,oquab2023dinov2} on a balanced
ImageNet-200 subset \cite{deng2009imagenet}, and ESM2-35M
\cite{lin2023evolutionary} on reviewed Swiss-Prot sequences
\cite{uniprot2025}. The action is exact zero-deactivation of selected MLP
features. Appendix~\ref{app:protocol} specifies the atlases, feature groups,
functionals, action-matched references, and uncertainty estimates.

The second application uses the atlas to localize MLP features for exact
zero-feature interventions without a backward pass for every query.

\ifarxiv\begin{figure}[!htb]\else\begin{figure}[t]\fi
  \centering
  \ifarxiv\includegraphics{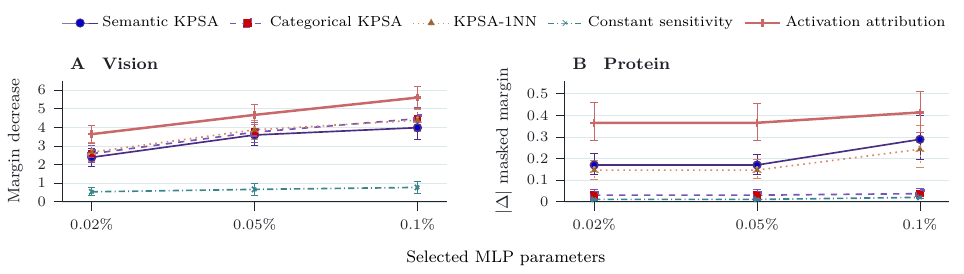}\else\includegraphics{figures/rendered/application1_primary.pdf}\fi
  \caption{\textbf{KPSA amortizes conditional neuron localization in vision and protein models.} Panels report the native-scale causal effect of zero-feature deactivation in vision and protein models. Vision measures the decrease in the correct-class logit margin; protein measures the absolute change in the true masked-residue margin. Each panel shows semantic KPSA, categorical KPSA, KPSA-1NN, constant sensitivity, and the per-query gradient-times-activation reference at 0.02\%, 0.05\%, and 0.1\% parameter budgets. The two y axes use different scales because the model functionals are not commensurate. ESM2's coarse grouping maps 0.02\% and 0.05\% to the same one-group action. Points are query means and bars are 95\% bootstrap intervals.}
  \label{fig:application1-primary}
\end{figure}

With the 800-source vision atlas, semantic KPSA recovers $66.5\%$ of the
causal gap at the $0.1\%$ budget; categorical KPSA and KPSA-1NN reach
$76.5\%$ and $74.8\%$, respectively.  On protein modeling, semantic KPSA is
the strongest approximate method at $68.3\%$, compared with $56.6\%$ for
KPSA-1NN and $4.4\%$ for the categorical kernel.  The contrast
reinforces the role of the query representation rather than supporting a
modality-independent best kernel.

\subsection{Accuracy--cost tradeoff}

\ifarxiv\begin{figure}[!htb]\else\begin{figure}[t]\fi
  \centering
  \ifarxiv\includegraphics{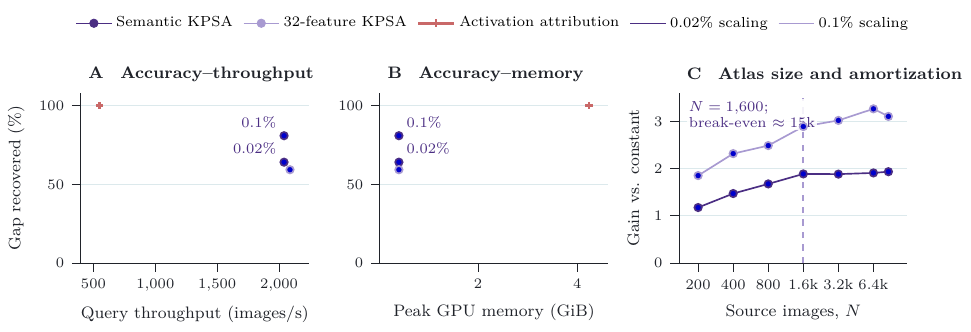}\else\includegraphics{figures/rendered/application1_efficiency.pdf}\fi
  \caption{\textbf{Amortized neuron localization occupies a distinct accuracy--cost region.} Panels A and B place semantic KPSA, a compact 32-feature-group atlas, and per-query activation attribution on measured causal fidelity versus batched throughput and peak GPU-memory axes on one RTX 5090. Labels identify KPSA's parameter budget; the feature-atlas accuracy points use the 6,400-source confirmation. Panel C shows causal gain over constant sensitivity as the number of source images increases; the vertical line marks the cost-balanced 1,600-source atlas, whose measured profile-only construction cost breaks even after approximately 15,000 batched queries.}
  \label{fig:application1-efficiency}
\end{figure}

At the 6,400-source confirmation operating point, feature-level KPSA recovers
$80.9\%$ of the activation-attribution gap while processing $2{,}041$
images/s, $3.7\times$ the reference throughput, and using $0.41$ rather than
$4.24$~GiB of peak allocated memory.  Its persistent atlas is 112.5~MiB.  A
32-feature partition reduces this to 3.5~MiB and retains $59.3\%$ gap
recovery.  Figure~\ref{fig:application1-efficiency}C shows the expected
finite-source saturation rather than a monotone guarantee for every source
set; the measured profile-only construction cost crosses the batched query
savings near 15,000 queries and excludes representation encoding and
prototype fitting.

\ifarxiv\else\FloatBarrier\fi

\section{Conclusion}
\label{sec:conclusion}

KPSA derives a simplex-valued parameter-sensitivity profile from the block
diagonal of the empirical NTK and estimates that profile conditionally in a
query space.  The resulting atlas has a precise population target.  Its
analysis separates finite-source, local-smoothing, kernel-approximation, and
irreducible representation errors, and gives exact top-$k$ recovery when the
uniform error is smaller than the reference margin.  Across the frozen
experiments, KPSA retains useful signals for sensitivity circuits and neuron
localization while moving query cost from backpropagation to a kernel query over a
compact atlas, providing one step toward scalable sensitivity-based
interpretability.

KPSA estimates sensitivity for a fixed model, functional, partition,
representation, and source distribution.
We do not test different, or more sophisticated, functionals, such as
task-specific behavioral scores or intervention-defined objectives used in
circuit discovery
\cite{conmy2023automated,syed2024attribution,marks2025sparse}.
It cannot recover variation discarded by the representation; fixed bandwidth
retains smoothing bias; normalized unsigned energy does not encode sign or
directed computation; top-$k$ recovery needs a margin; and the perturbation
bridge is local in parameter space.  Distribution shift, adaptive bandwidths,
and signed or interaction-aware atlases remain open statistical problems.

\subsubsection*{AI Use Statement}

Generative AI tools were used to assist with problem formulation and
conceptual development; hypothesis, methodology, and experiment design;
theoretical framing and mathematical derivations; method and evaluation-code
implementation; experiment execution; analysis
and interpretation of results; figure creation; literature search and
synthesis; and manuscript structure, drafting, and editing. The authors
reviewed AI-assisted code, experimental outputs, numerical claims, citations,
figures, and text against saved artifacts and executable code, and take
responsibility for the final content.

\bibliographystyle{plain}
\bibliography{references}

\begin{thebibliography}{10}

\bibitem{achille2019task2vec}
Alessandro Achille, Michael Lam, Rahul Tewari, Avinash Ravichandran, Subhransu
  Maji, Charless~C. Fowlkes, Stefano Soatto, and Pietro Perona.
\newblock {Task2Vec}: Task embedding for meta-learning.
\newblock In {\em IEEE/CVF International Conference on Computer Vision}, pages
  6429--6438, 2019.

\bibitem{achille2021information}
Alessandro Achille, Giovanni Paolini, Glen Mbeng, and Stefano Soatto.
\newblock The information complexity of learning tasks, their structure and
  their distance.
\newblock {\em Information and Inference: A Journal of the IMA}, 10(1):51--72,
  2021.

\bibitem{ansari2024chronos}
Abdul~Fatir Ansari, Lorenzo Stella, Caner Turkmen, et~al.
\newblock Chronos: Learning the language of time series.
\newblock {\em Transactions on Machine Learning Research}, 2024.

\bibitem{arora2026language}
Aryaman Arora, Zhengxuan Wu, Jacob Steinhardt, and Sarah Schwettmann.
\newblock Language model circuits are sparse in the neuron basis.
\newblock {\em arXiv preprint arXiv:2601.22594}, 2026.

\bibitem{ash2020deep}
Jordan~T. Ash, Chicheng Zhang, Akshay Krishnamurthy, John Langford, and Alekh
  Agarwal.
\newblock Deep batch active learning by diverse, uncertain gradient lower
  bounds.
\newblock In {\em International Conference on Learning Representations}, 2020.

\bibitem{cai2022fast}
Difeng Cai, James Nagy, and Yuanzhe Xi.
\newblock Fast deterministic approximation of symmetric indefinite kernel
  matrices with high dimensional datasets.
\newblock {\em SIAM Journal on Matrix Analysis and Applications},
  43(2):1003--1028, 2022.

\bibitem{chatalic2022nystrom}
Antoine Chatalic, Nicolas Schreuder, Lorenzo Rosasco, and Alessandro Rudi.
\newblock Nystr{\"o}m kernel mean embeddings.
\newblock In {\em Proceedings of the 39th International Conference on Machine
  Learning}, volume 162 of {\em Proceedings of Machine Learning Research},
  pages 3006--3024. PMLR, 2022.

\bibitem{chen2026mechanistic}
Jianhui Chen, Yuzhang Luo, and Liangming Pan.
\newblock Mechanistic data attribution: Tracing the training origins of
  interpretable {LLM} units.
\newblock {\em arXiv preprint arXiv:2601.21996}, 2026.

\bibitem{cloud2024gradient}
Alex Cloud, Jacob Goldman-Wetzler, Ev{\v{z}}en Wybitul, Joseph Miller, and
  Alexander~Matt Turner.
\newblock Gradient routing: Masking gradients to localize computation in neural
  networks.
\newblock {\em arXiv preprint arXiv:2410.04332}, 2024.

\bibitem{conmy2023automated}
Arthur Conmy, Augustine~N. Mavor-Parker, Aengus Lynch, Stefan Heimersheim, and
  Adri\`a Garriga-Alonso.
\newblock Towards automated circuit discovery for mechanistic interpretability.
\newblock In {\em Advances in Neural Information Processing Systems},
  volume~36, 2023.

\bibitem{dai2022knowledge}
Damai Dai, Li~Dong, Yaru Hao, Zhifang Sui, Baobao Chang, and Furu Wei.
\newblock Knowledge neurons in pretrained transformers.
\newblock In {\em Proceedings of the 60th Annual Meeting of the Association for
  Computational Linguistics}, pages 8493--8502, 2022.

\bibitem{deng2009imagenet}
Jia Deng, Wei Dong, Richard Socher, Li-Jia Li, Kai Li, and Li~Fei-Fei.
\newblock {ImageNet}: A large-scale hierarchical image database.
\newblock In {\em IEEE Conference on Computer Vision and Pattern Recognition},
  pages 248--255, 2009.

\bibitem{dosovitskiy2021image}
Alexey Dosovitskiy, Lucas Beyer, Alexander Kolesnikov, et~al.
\newblock An image is worth 16x16 words: Transformers for image recognition at
  scale.
\newblock In {\em International Conference on Learning Representations}, 2021.

\bibitem{foote2024tackling}
Alex Foote.
\newblock Tackling polysemanticity with neuron embeddings.
\newblock {\em arXiv preprint arXiv:2411.08166}, 2024.

\bibitem{grunewalder2012conditional}
Steffen Gr{\"u}new{\"a}lder, Guy Lever, Luca Baldassarre, Sam Patterson, Arthur
  Gretton, and Massimiliano Pontil.
\newblock Conditional mean embeddings as regressors.
\newblock In {\em Proceedings of the 29th International Conference on Machine
  Learning}, pages 1823--1830. Omnipress, 2012.

\bibitem{jaakkola1999exploiting}
Tommi~S. Jaakkola and David Haussler.
\newblock Exploiting generative models in discriminative classifiers.
\newblock In {\em Advances in Neural Information Processing Systems},
  volume~11, pages 487--493, 1999.

\bibitem{jacot2018neural}
Arthur Jacot, Franck Gabriel, and Cl{\'e}ment Hongler.
\newblock Neural tangent kernel: Convergence and generalization in neural
  networks.
\newblock In {\em Advances in Neural Information Processing Systems},
  volume~31, 2018.

\bibitem{jethani2022fastshap}
Neil Jethani, Mukund Sudarshan, Ian Covert, Su-In Lee, and Rajesh Ranganath.
\newblock {FastSHAP}: Real-time shapley value estimation.
\newblock In {\em International Conference on Learning Representations}, 2022.

\bibitem{kirkpatrick2017overcoming}
James Kirkpatrick, Razvan Pascanu, Neil Rabinowitz, Joel Veness, Guillaume
  Desjardins, Andrei~A. Rusu, Kieran Milan, John Quan, Tiago Ramalho, Agnieszka
  Grabska-Barwinska, Demis Hassabis, Claudia Clopath, Dharshan Kumaran, and
  Raia Hadsell.
\newblock Overcoming catastrophic forgetting in neural networks.
\newblock {\em Proceedings of the National Academy of Sciences},
  114(13):3521--3526, 2017.

\bibitem{koh2017understanding}
Pang~Wei Koh and Percy Liang.
\newblock Understanding black-box predictions via influence functions.
\newblock In {\em Proceedings of the 34th International Conference on Machine
  Learning}, volume~70 of {\em Proceedings of Machine Learning Research}, pages
  1885--1894. PMLR, 2017.

\bibitem{leng2025towards}
Yongqi Leng and Deyi Xiong.
\newblock Towards understanding multi-task learning (generalization) of {LLM}s
  via detecting and exploring task-specific neurons.
\newblock In {\em Proceedings of the 31st International Conference on
  Computational Linguistics}, pages 2969--2987, 2025.

\bibitem{lin2023evolutionary}
Zeming Lin, Halil Akin, Roshan Rao, et~al.
\newblock Evolutionary-scale prediction of atomic-level protein structure with
  a language model.
\newblock {\em Science}, 379(6637):1123--1130, 2023.

\bibitem{liu2025wasserstein}
Xinran Liu, Yikun Bai, Yuzhe Lu, Andrea Soltoggio, and Soheil Kolouri.
\newblock Wasserstein task embedding for measuring task similarities.
\newblock {\em Neural Networks}, 181:106796, 2025.

\bibitem{marks2025sparse}
Samuel Marks, Can Rager, Eric~J. Michaud, Yonatan Belinkov, David Bau, and
  Aaron Mueller.
\newblock Sparse feature circuits: Discovering and editing interpretable causal
  graphs in language models.
\newblock In {\em International Conference on Learning Representations}, 2025.

\bibitem{muandet2017kernel}
Krikamol Muandet, Kenji Fukumizu, Bharath Sriperumbudur, and Bernhard
  Sch{\"o}lkopf.
\newblock Kernel mean embedding of distributions: A review and beyond.
\newblock {\em Foundations and Trends in Machine Learning}, 10(1--2):1--141,
  2017.

\bibitem{mueller2025mib}
Aaron Mueller, Atticus Geiger, Sarah Wiegreffe, Dana Arad, Iv{\'a}n Arcuschin,
  Adam Belfki, Yik~Siu Chan, Jaden~Fried Fiotto-Kaufman, Tal Haklay, Michael
  Hanna, Jing Huang, Rohan Gupta, Yaniv Nikankin, Hadas Orgad, Nikhil Prakash,
  Anja Reusch, Aruna Sankaranarayanan, Shun Shao, Alessandro Stolfo, Martin
  Tutek, Amir Zur, David Bau, and Yonatan Belinkov.
\newblock {MIB}: A mechanistic interpretability benchmark.
\newblock In {\em Proceedings of the 42nd International Conference on Machine
  Learning}, volume 267 of {\em Proceedings of Machine Learning Research},
  pages 45069--45108. PMLR, 2025.

\bibitem{nadaraya1964estimating}
Elizbar~A. Nadaraya.
\newblock On estimating regression.
\newblock {\em Theory of Probability and Its Applications}, 9(1):141--142,
  1964.

\bibitem{oglic2017nystrom}
Dino Oglic and Thomas G{\"a}rtner.
\newblock Nystr{\"o}m method with kernel k-means++ samples as landmarks.
\newblock In {\em Proceedings of the 34th International Conference on Machine
  Learning}, volume~70 of {\em Proceedings of Machine Learning Research}, pages
  2652--2660. PMLR, 2017.

\bibitem{oquab2023dinov2}
Maxime Oquab, Timoth{\'e}e Darcet, Th{\'e}o Moutakanni, et~al.
\newblock {DINOv2}: Learning robust visual features without supervision.
\newblock {\em Transactions on Machine Learning Research}, 2024.

\bibitem{park2022neurocartography}
Haekyu Park, Nilaksh Das, Rahul Duggal, Austin~P. Wright, Omar Shaikh, Fred
  Hohman, and Duen Horng~Polo Chau.
\newblock {NeuroCartography}: Scalable automatic visual summarization of
  concepts in deep neural networks.
\newblock {\em IEEE Transactions on Visualization and Computer Graphics},
  28(1):813--823, 2022.

\bibitem{pham2013fast}
Ninh Pham and Rasmus Pagh.
\newblock Fast and scalable polynomial kernels via explicit feature maps.
\newblock In {\em Proceedings of the 19th ACM SIGKDD International Conference
  on Knowledge Discovery and Data Mining}, pages 239--247, 2013.

\bibitem{pruthi2020estimating}
Garima Pruthi, Frederick Liu, Satyen Kale, and Mukund Sundararajan.
\newblock Estimating training data influence by tracing gradient descent.
\newblock In {\em Advances in Neural Information Processing Systems},
  volume~33, 2020.

\bibitem{smola2007hilbert}
Alex Smola, Arthur Gretton, Le~Song, and Bernhard Sch{\"o}lkopf.
\newblock A hilbert space embedding for distributions.
\newblock In {\em Algorithmic Learning Theory}, volume 4754 of {\em Lecture
  Notes in Computer Science}, pages 13--31. Springer, 2007.

\bibitem{song2024large}
Ran Song, Shizhu He, Shuting Jiang, Yantuan Xian, Shengxiang Gao, Kang Liu, and
  Zhengtao Yu.
\newblock Does large language model contain task-specific neurons?
\newblock In {\em Proceedings of the 2024 Conference on Empirical Methods in
  Natural Language Processing}, pages 7101--7113, 2024.

\bibitem{steiner2022train}
Andreas Steiner, Alexander Kolesnikov, Xiaohua Zhai, Ross Wightman, Jakob
  Uszkoreit, and Lucas Beyer.
\newblock How to train your {ViT}? data, augmentation, and regularization in
  vision transformers.
\newblock {\em Transactions on Machine Learning Research}, 2022.

\bibitem{sweeney2026geometry}
John Sweeney.
\newblock The geometry of updates: Fisher alignment at vocabulary scale.
\newblock In {\em Proceedings of the 43rd International Conference on Machine
  Learning}, 2026.

\bibitem{syed2024attribution}
Aaquib Syed, Can Rager, and Arthur Conmy.
\newblock Attribution patching outperforms automated circuit discovery.
\newblock In {\em Proceedings of the 7th BlackboxNLP Workshop}, pages 407--416,
  2024.

\bibitem{uniprot2025}
{The UniProt Consortium}.
\newblock {UniProt}: The universal protein knowledgebase in 2025.
\newblock {\em Nucleic Acids Research}, 53(D1):D609--D617, 2025.

\bibitem{tsybakov2009introduction}
Alexandre~B. Tsybakov.
\newblock {\em Introduction to Nonparametric Estimation}.
\newblock Springer Series in Statistics. Springer, New York, 2009.

\bibitem{wang2026heads}
Qidong Wang, Junjie Hu, and Ming Jiang.
\newblock From heads to neurons: Causal attribution and steering in multi-task
  vision--language models.
\newblock In {\em Findings of the Association for Computational Linguistics:
  ACL 2026}, pages 36151--36175, 2026.

\bibitem{watson1964smooth}
Geoffrey~S. Watson.
\newblock Smooth regression analysis.
\newblock {\em Sankhy{\=a}: The Indian Journal of Statistics, Series A},
  26(4):359--372, 1964.

\bibitem{wiest2026improved}
Davide Wiest.
\newblock Improved elastic weight consolidation as an optimization constraint
  for continual learning.
\newblock {\em arXiv preprint arXiv:2611.16039}, 2026.

\bibitem{wightman2019timm}
Ross Wightman.
\newblock Pytorch image models.
\newblock GitHub repository, 2019.

\bibitem{wu2021autoformer}
Haixu Wu, Jiehui Xu, Jianmin Wang, and Mingsheng Long.
\newblock Autoformer: Decomposition transformers with auto-correlation for
  long-term series forecasting.
\newblock In {\em Advances in Neural Information Processing Systems},
  volume~34, 2021.

\bibitem{wu2024improved}
Xiaodong Wu, Wenyi Yu, Chao Zhang, and Philip Woodland.
\newblock An improved empirical fisher approximation for natural gradient
  descent.
\newblock In {\em Advances in Neural Information Processing Systems},
  volume~37, 2024.

\bibitem{zhang2022neural}
Ruixiang Zhang, Shuangfei Zhai, Etai Littwin, and Joshua~M. Susskind.
\newblock Learning representation from neural fisher kernel with low-rank
  approximation.
\newblock In {\em International Conference on Learning Representations}, 2022.

\bibitem{zhou2021informer}
Haoyi Zhou, Shanghang Zhang, Jieqi Peng, Shuai Zhang, Jianxin Li, Hui Xiong,
  and Wancai Zhang.
\newblock Informer: Beyond efficient transformer for long sequence time-series
  forecasting.
\newblock In {\em AAAI Conference on Artificial Intelligence}, volume~35, pages
  11106--11115, 2021.

\end{thebibliography}

\ifarxiv\else
\input{sections/reproducibility_checklist}
\fi

\clearpage
\appendix
\ifarxiv\else
\onecolumn
\fi
\setcounter{figure}{0}
\renewcommand{\thefigure}{A\arabic{figure}}
\setcounter{table}{0}
\renewcommand{\thetable}{A\arabic{table}}
\section{Experimental Protocol}
\label{app:protocol}

All experiments use frozen pretrained models in evaluation mode.  Source
examples are used to build an atlas; evaluation queries are disjoint from
those sources.  KPSA does not use a query backward pass.  We nevertheless
compute query gradients during evaluation to define the direct references and
to measure how much first-order sensitivity the retrieved groups contain.

\subsection{Atlas construction and query rules}

Let $F_\theta(x)$ be the scalar functional used in an experiment and let
$\mathcal P$ be the full parameter-group partition used to form a profile.
The groups eligible for a downstream action form a subset
$\mathcal G\subseteq\mathcal P$. For each source input $x_i$, we compute the
squared gradient energy of group $g$ and normalize it over $\mathcal P$:
\begin{equation}
  e_{ig}=\lVert \nabla_{\theta_g}F_\theta(x_i)\rVert_2^2,
  \qquad
  r_{ig}=\frac{e_{ig}}{\sum_{h\in\mathcal P} e_{ih}}.
  \label{eq:protocol-profile}
\end{equation}
Thus $\mathbf r_i$ records where the profiled sensitivity lies for source
$i$; it does not retain the total gradient norm. The ViT profile covers the
complete model partition, whereas the Chronos and ESM2 profiles cover all
feed-forward groups of the chosen component. We retain the coordinates in
$\mathcal G$ for ranking and intervention. All paper results use this
normalized profile.

For a query representation $z_*$, the exact kernel score for group $g$ is
\begin{equation}
  \widehat r_g(z_*)=\frac{1}{N}\sum_{i=1}^{N}
  k(z_i,z_*)r_{ig}.
  \label{eq:protocol-query}
\end{equation}
When $k(z,z')=\langle\phi(z),\phi(z')\rangle$, we store
$\widehat{\mathbf m}_g=N^{-1}\sum_i r_{ig}\phi(z_i)$ and evaluate
$\langle\widehat{\mathbf m}_g,\phi(z_*)\rangle$.  The stored atlas therefore
has one $D$-dimensional row for every parameter group, not one row in total.
Normalizing Equation~\ref{eq:protocol-query} by $\sum_i k(z_i,z_*)$ would
multiply every group score for one query by the same positive number and
would not change the selected groups.

\paragraph{Query spaces.}
Learned-representation KPSA uses an $\ell_2$-normalized hidden
representation and an RBF kernel.  The ImageNet experiments approximate this
kernel with fixed prototype-response features; the forecasting and protein
experiments evaluate it against all source representations.  Categorical
KPSA uses a delta kernel on class identity in vision, a joint
hour/day-of-week/month coordinate in forecasting, and masked-residue identity
in protein modeling.  Categorical KPSA is a second instance of the method,
not an external baseline.

\paragraph{Comparison methods.}
KPSA-1NN ablates kernel aggregation by returning the normalized profile of
the source with the largest representation cosine similarity.  Constant sensitivity returns the mean
source profile for every query.  A matched-shuffle control preserves the
source representations and profiles but permutes their pairing; it is used
during method selection and in the circuit-overlap analysis.  The direct
parameter reference computes Equation~\ref{eq:protocol-profile} on each
query.  The action-matched reference for neuron localization instead ranks
features by their first-order gradient-times-activation score for the exact
zero-feature intervention.

Groups are ranked by their nonnegative score.  We add groups in descending
order until reaching the requested fraction of the eligible parameter mass.
The main figures use 0.02\%, 0.05\%, and 0.1\%.  These correspond to 8, 19,
and 37 ViT MLP features; 1, 2, and 3 features in the selected Chronos block;
and 1, 1, and 2 features in the selected ESM2 block.  The repeated ESM2 count
at the first two budgets follows from its coarser groups.

\subsection{Models, data, and frozen settings}

Table~\ref{tab:protocol-settings} lists the instances used in the main
figures.  $N$ is the number of atlas sources and $n$ is the number of
evaluation queries entering the reported mean.  For classification and masked
prediction we retain only queries the intact model answers correctly.  No
such filter is used for forecasting.

\begin{table}[H]
  \centering
  \footnotesize
  \setlength{\tabcolsep}{3pt}
  \renewcommand{\arraystretch}{1.05}
  \begin{tabularx}{\linewidth}{@{}>{\raggedright\arraybackslash}p{.09\linewidth}>{\raggedright\arraybackslash}p{.12\linewidth}>{\raggedright\arraybackslash}p{.125\linewidth}>{\raggedright\arraybackslash}p{.18\linewidth}>{\raggedright\arraybackslash}Xrr@{}}
    \toprule
    Use & Setting & Model & Functional $F$ & Eligible groups and query kernel & $N$ & $n$ \\
    \midrule
    Sensitivity & ImageNet-1k & ViT-B/16 (86.6M) & True-class logit minus strongest alternative & 36,864 coupled MLP features; DINOv2-S CLS, 1,000-prototype RBF map, scale .10 & 8,000 & 169 \\
    Sensitivity & ETTh1 & Chronos-Bolt-small (47.7M) & Mean 24-step median forecast minus last observation & 2,048 final-decoder FF features; final encoder state, exact RBF, scale .25 & 1,600 & 128 \\
    Sensitivity & Weather & Chronos-Bolt-small (47.7M) & Same forecasting functional & Same frozen ETTh1 setting & 1,600 & 128 \\
    Sensitivity & Swiss-Prot & ESM2-35M (33.5M) & True masked-residue logit minus strongest alternative & 1,920 layer-3 outgoing FF features; sequence mean, exact RBF, scale .10 & 1,600 & 48 \\
    \addlinespace[2pt]
    Circuit & ImageNet-1k & ViT-B/16 & ImageNet logit margin & Full 12-block MLP stack; DINOv2-S CLS, 1,000-prototype RBF map, scale .10 & 8,000 & 110 \\
    Circuit & ETTh1 & Chronos-Bolt-small & Forecasting functional above & Full 12-block FF stack; final encoder state, exact RBF, scale .25 & 1,600 & 128 \\
    \addlinespace[2pt]
    Neuron loc. & ImageNet-200 & ViT-B/16 & ImageNet logit margin & 36,864 coupled MLP features; DINOv2-S CLS, 800-prototype RBF map, scale .025 & 800 & 76 \\
    Neuron loc. & Swiss-Prot & ESM2-35M & Masked-residue margin & Same ESM2 atlas and scope as above & 1,600 & 128 \\
    \bottomrule
  \end{tabularx}
  \caption{Frozen experimental instances used in the primary causal and circuit figures. ImageNet-200 is a balanced 200-class subset of ImageNet validation. The circuit-structure rows use a wider parameter scope than the corresponding primary causal rows; the separate efficiency variants are specified in Section~\ref{app:protocol-efficiency}.}
  \label{tab:protocol-settings}
\end{table}

\paragraph{Vision.}
The ImageNet-1k parameter sensitivity atlas uses eight source images per class
from ImageNet validation~\cite{deng2009imagenet}.
Separate images at within-class offset 43 construct the prototype map, and
offset 45 supplies the confirmation pool.  We sample 200 evenly spaced
classes from that pool and report the 169 correctly classified queries. The
circuit-overlap analysis uses offsets 46--48. Its exact causal subset uses
offset 46 and reports 110 correctly classified queries out of 128 sampled.
The neuron-localization study uses a frozen balanced 200-class subset, 800
source images, 800 disjoint prototype-construction images, and a fresh target
shift.  Of 100 sampled queries, 76 are classified correctly.
DINOv2-S~\cite{oquab2023dinov2} is used only to define the query space;
ViT-B/16~\cite{dosovitskiy2021image} supplies the functional and the parameter
gradients.

Each ViT feature group couples one row of the first MLP projection, its bias,
and the corresponding column of the second projection.  This gives 3,072
groups in each of 12 blocks.  Classifier, attention, normalization, and patch
embedding parameters are outside the eligible scope.

\paragraph{Forecasting.}
We use univariate ETTh1 oil temperature~\cite{zhou2021informer} and Weather
air temperature~\cite{wu2021autoformer}. Each query contains 512 observations
and Chronos-Bolt-small~\cite{ansari2024chronos} predicts 24 steps. Values
are standardized through index 8,400 for ETTh1 and 24,000 for Weather. Source
window endpoints lie in $[600,8400]$ for ETTh1 and $[600,24000]$ for Weather;
query endpoints lie in $[12300,16800]$ and $[33000,50000]$, respectively.
The 1,600 source endpoints and 128 query endpoints are evenly spaced in their
ranges.  A feature group couples an incoming FF row, its bias, and the
corresponding outgoing column.  The final decoder block, representation, and
RBF scale were selected on ETTh1 development data and then frozen for the
later ETTh1 confirmation and the Weather transfer.

\paragraph{Protein modeling.}
The protein corpus contains 3,000 reviewed Swiss-Prot
sequences~\cite{uniprot2025} of length 80--400. One canonical amino acid per
sequence is masked at a position fixed
by a hash of its accession.  Source and query accessions do not overlap, and
queries are retained only when ESM2 predicts the masked residue correctly.
The query representation is the normalized mean final hidden state over
non-special sequence tokens. We use the 12-layer, 35M-parameter ESM2
checkpoint~\cite{lin2023evolutionary}. A group is one outgoing column of the layer-3
feed-forward output projection.  The representation, outgoing component,
layer, and RBF scale were chosen in a bounded development screen and then
held fixed.  The same 1,600-source atlas is used for both applications.

\subsection{Application 2: exact parameter-perturbation assay}

For a selected set $A$ and Rademacher vector $\xi_d$ supported on its
parameters, we measure
\begin{equation}
  C_A(x)=\frac{1}{D}\sum_{d=1}^{D}
  \left(
    \frac{F_{\theta+\sigma\xi_d}(x)-F_{\theta-\sigma\xi_d}(x)}{2\sigma}
  \right)^2.
  \label{eq:protocol-susceptibility}
\end{equation}
This is an exact-forward central-difference estimate of local squared
susceptibility along the sampled parameter directions. We set $\sigma$ to
3\% of the scoped parameter root-mean-square value in vision and protein and
to 3\% of the all-FF root-mean-square value in Chronos. The ImageNet
primary experiment uses 32 antithetic directions per query, method, and
budget; the circuit, forecasting, and protein experiments use 16.  Directions
are averaged within a query and are never treated as independent samples.
The same seeded group-coordinate signs are used across methods where their
selected groups overlap.

The direct-gradient reference ranks groups by their query gradient energy.
It is the exact first-order score for the stated functional and group
decomposition; the finite perturbation assay remains necessary because
Equation~\ref{eq:protocol-susceptibility} also contains local nonlinear
effects. For each method we first sum the direct query-gradient energy in its
selected set. The plotted gradient-energy value divides this sum by the
energy in the same-budget direct-gradient selection. It therefore reports
recovery of the best first-order group selection at that budget, rather than
a percentage of the model's total gradient energy. Causal gap recovery is
\begin{equation}
  \frac{\overline C_{\mathrm{method}}-
        \overline C_{\mathrm{constant}}}
       {\overline C_{\mathrm{direct}}-
        \overline C_{\mathrm{constant}}}.
  \label{eq:protocol-causal-gap}
\end{equation}

For Figure~\ref{fig:application2-circuits}, we additionally select from the
full 12-block stack at 0.1\% and 0.5\%.  The matched constant control receives
the same number of groups in every layer as the KPSA selection.  The
``first-nine'' and ``final-quarter'' controls restrict both methods to blocks
1--9 or 10--12 before matching layer counts.  Circuit overlap is the Jaccard
coefficient between two selected group sets.  The shuffled-pairing control
uses the same query representations and source profiles as KPSA, but breaks
their correspondence.

\subsection{Application 1: exact feature-deactivation assay}

For neuron localization, the selected features are set to zero at the input
to the feed-forward output projection, at every token or patch position.  In
vision the causal effect is the signed decrease in the correct-class logit
margin.  In protein modeling it is the absolute change in the masked-residue
margin.  The latter matches the absolute gradient-times-activation ranking
used for that task.  The vision reference uses the positive first-order
degradation $\nabla_a F\cdot a$; the protein reference sums
$|\nabla_a F\cdot a|$ over tokens.  The reported attribution-gap recovery
replaces $C$ in Equation~\ref{eq:protocol-causal-gap} with the corresponding
exact zero-deactivation effect and replaces the direct parameter reference
with this action-matched activation-attribution reference.

The vision prototype count and bandwidth were chosen on a development split
by causal gain over constant sensitivity, with a requirement that the method
also beat its matched-shuffle control at the tested budgets.  The reported
confirmation uses a fresh target shift.  The protein setting was frozen after
the bounded representation, component, layer, and bandwidth screen described
above.

\subsection{Uncertainty and efficiency measurement}
\label{app:protocol-efficiency}

Primary-figure intervals use 5,000 paired nonparametric bootstrap resamples
of queries.  For each resample we recompute the ratio of query means in
Equations~\ref{eq:protocol-causal-gap}; we do not average per-query ratios or
clip intervals to $[0,100]$.  Circuit causal comparisons first average
perturbation directions within each query and then use paired 95\% $t$
intervals across queries.  Circuit-overlap intervals use the input pairs in
each similarity bin.  We report estimation intervals for the fixed models,
datasets, and seeds; they do not measure variation from model training.

The vision efficiency benchmark runs on one NVIDIA RTX 5090 with PyTorch
2.13.0 and CUDA 13.0.  Batched results use batch size 32, five warm-up passes,
and ten timed batches.  A KPSA query includes CPU image preprocessing, the
DINOv2-S forward pass, 800 prototype responses, the stored-atlas matrix
product, and top-$k$ selection.  The activation-attribution implementation
requests gradients only for the intermediate activations needed by its score.
Peak allocated GPU memory and persistent atlas storage are reported
separately.  Panels A--B of Figure~\ref{fig:application1-efficiency} use the
$N=6{,}400$ confirmation accuracy record; Panel C shows the independent
source-count sweep and its $N=1{,}600$ cost-balanced point.  Profile-only
break-even estimates exclude source representation encoding and prototype
fitting, and are therefore lower bounds on full construction amortization.

\subsection{Assets, licenses, and research ethics}
\label{app:assets}

All pretrained models and datasets are attributed to their creators in
Table~\ref{tab:protocol-settings} and the surrounding protocol.  The AugReg
ViT checkpoint is accessed through \texttt{timm}
\cite{wightman2019timm}; its code is Apache-2.0 and its ImageNet-trained
weights retain the upstream ImageNet-use caveat.  DINOv2 code and weights are
Apache-2.0, Chronos code and weights are Apache-2.0, and the ESM release is
MIT-licensed.  ImageNet is used under its non-commercial research terms, and
UniProt's copyrightable database content is CC BY 4.0.  ETT and Weather are
used from the cited public research releases under their upstream terms.

The supplement does not redistribute raw datasets or pretrained weights.  It
contains derived aggregate tables, figure-generation scripts, rendered paper
figures, and illustrative overview thumbnails.  No new dataset or model
checkpoint is released.  The experiments involve no crowdsourcing or human
subjects, and no identity, demographic, or sensitive-attribute inference is
performed.  ImageNet can contain people incidentally; the study uses only
model inputs and class labels for the stated non-commercial research purpose.

\section{Theoretical Assumptions and Proofs}
\label{app:theory}

\subsection{Formal setting}
\label{app:theory-setting}

Let $(X_i)_{i=1}^N$ be independent draws from a distribution $P_X$.  The
frozen functional $F_\theta$, parameter partition $\mathcal P$, query map
$\psi$, and retrieval kernel $\kappa_h$ are fixed independently of these
draws.  Write $Z_i=\psi(X_i)$ and $S_i=s(X_i)$, where $s$ is the normalized
block-NTK profile in Equation~\ref{eq:kpsa-profile}.  Positivity of the
profiled NTK diagonal gives
\begin{equation}
 S_i\in\Delta_M,\qquad \|S_i\|_1=1,\qquad \|S_i\|_2\leq1.
 \label{eq:app-simplex}
\end{equation}
For a stored subcollection of groups, the same arguments hold with
$\|S_i\|_1\leq1$.  The conditional target
$m_\psi(z)=\mathbb E[S\mid Z=z]$ is defined $P_Z$-almost everywhere and lies
in the same simplex or sub-simplex.

\subsection{Proof of the block-NTK derivation}

\begin{proof}[Proof of Proposition~\ref{prop:block-ntk-sensitivity}]
Because the parameter groups are disjoint, the full profiled gradient is the
concatenation of $(g_m)_{m=1}^M$.  Its inner product therefore decomposes as
\begin{equation*}
 \left\langle\nabla_{\theta_{\mathcal P}}F_\theta(x),
 \nabla_{\theta_{\mathcal P}}F_\theta(x')\right\rangle
 =\sum_{m=1}^M\langle g_m(x),g_m(x')\rangle,
\end{equation*}
which proves Equation~\ref{eq:block-ntk}.  On the diagonal this becomes a sum
of nonnegative squared norms, so Equation~\ref{eq:kpsa-profile} lies in the
simplex.

For a Rademacher direction supported on $G_m$,
\begin{equation*}
 \left.\frac{d}{dt}F_{\theta+t\xi_m}(x)\right|_{t=0}
 =g_m(x)^\top\xi_m.
\end{equation*}
Independence, zero mean, and unit variance of the coordinates give
$\mathbb E[(g_m(x)^\top\xi_m)^2]=\|g_m(x)\|_2^2=K_m(x,x)$.  For equal-mass
groups, maximizing the sum of these nonnegative quantities over $k$ groups
selects their $k$ largest values, equivalently the $k$ largest coordinates of
$s(x)$.
\end{proof}

\subsection{Proof of query-space sensitivity sufficiency}

\begin{proof}[Proof of Proposition~\ref{prop:query-sufficiency}]
Write $m(Z)=m_\psi(Z)$ and expand
\begin{align*}
 \|S-f(Z)\|_2^2
 &=\|S-m(Z)\|_2^2+\|m(Z)-f(Z)\|_2^2\\
 &\quad+2\langle S-m(Z),m(Z)-f(Z)\rangle.
\end{align*}
The conditional expectation of the last term given $Z$ is zero because
$\mathbb E[S-m(Z)\mid Z]=0$.  Taking expectations proves
Equation~\ref{eq:query-sufficiency}.  The lower bound follows because the
second term is nonnegative.  Finally, $\mathcal I_\psi=0$ if and only if
$S=m_\psi(Z)$ almost surely, which is equivalent to the sensitivity profile
being determined by the query representation up to null sets.
\end{proof}

\subsection{Proof of the finite-source sensitivity gap}

\begin{proof}[Proof of Theorem~\ref{thm:truth-gap}]
Fix the source representations and abbreviate
$\lambda_i=\lambda_i(z_*)$ and $m_i=m_\psi(Z_i)$.  Decompose
\begin{equation}
 \widehat m(z_*)-m_\psi(z_*)
 =\underbrace{\sum_i\lambda_i(S_i-m_i)}_{V}
 +\underbrace{\sum_i\lambda_i(m_i-m_\psi(z_*))}_{B}.
 \label{eq:app-bias-variance}
\end{equation}
Conditional independence and $\mathbb E[S_i-m_i\mid Z_i]=0$ imply
\begin{equation}
 \mathbb E[\|V\|_2^2\mid Z_{1:N}]
 =\sum_i\lambda_i^2
   \mathbb E[\|S_i-m_i\|_2^2\mid Z_i]
 \leq\sum_i\lambda_i^2.
 \label{eq:app-vector-variance}
\end{equation}
Hence
$\mathbb E[\|V\|_2\mid Z_{1:N}]\leq n_{\mathrm{eff}}^{-1/2}$.
Replacing one profile by another simplex point changes $\|V\|_2$ by at most
$\sqrt2\lambda_i$.  Conditional McDiarmid concentration gives, with
probability at least $1-\delta$,
\begin{equation*}
 \|V\|_2\leq
 \frac{1+\sqrt{\log(1/\delta)}}{\sqrt{n_{\mathrm{eff}}}}.
\end{equation*}
Target stability gives
\begin{equation*}
 \|B\|_2
 \leq\sum_i\lambda_i\|m_\psi(Z_i)-m_\psi(z_*)\|_2
 \leq L\rho_\alpha(z_*).
\end{equation*}
Finally, $\|u\|_\infty\leq\|u\|_2$ and
$\|m_\psi(z_*)-s(x_*)\|_\infty=\varepsilon_\psi(x_*)$ give
Equation~\ref{eq:truth-gap}.
\end{proof}

\paragraph{Categorical kernels.}
For a delta kernel, every nonzero-weight source shares the query category.
The stability term is zero, $n_{\mathrm{eff}}$ is the category source count
under uniform weights, and the representation residual records
within-category sensitivity heterogeneity.

\subsection{Proof of the local rate}

\begin{proof}[Proof of Corollary~\ref{cor:local-rate}]
Let $D=d(Z,z_*)$ and suppose that, for all sufficiently small $r$,
\begin{equation}
 c_-r^d\leq\Pr(D\leq r)\leq c_+r^d
 \label{eq:app-small-ball}
\end{equation}
for constants $c_-,c_+>0$.  Take
$\kappa_h(Z,z_*)=\exp[-D^2/(2h^2)]$ and let
$B_h=\sum_i\mathbf1\{d(Z_i,z_*)\leq h\}$.  A binomial Chernoff bound and the
lower small-ball condition give, with probability tending to one whenever
$Nh^d\to\infty$,
\begin{equation}
 \sum_i\kappa_h(Z_i,z_*)\geq e^{-1/2}B_h\geq cNh^d
 \label{eq:app-denominator}
\end{equation}
for a constant $c>0$.  Since $0\leq\kappa_h\leq1$,
\begin{equation}
 n_{\mathrm{eff}}
 =\frac{(\sum_i\kappa_i)^2}{\sum_i\kappa_i^2}
 \geq\sum_i\kappa_i=\Omega_p(Nh^d).
 \label{eq:app-neff-rate}
\end{equation}
Splitting the query space into shells of width $h$ and applying the upper
small-ball condition gives
$\mathbb E[\kappa_h(Z,z_*)D^\alpha]=O(h^{d+\alpha})$.
Consequently,
$\sum_i\kappa_iD_i^\alpha=O_p(Nh^{d+\alpha})$ and
$\rho_\alpha=O_p(h^\alpha)$.  Combining these bounds with the proof of
Theorem~\ref{thm:truth-gap}, now targeting $m_\psi(z_*)$ rather than a
realized profile, yields
\begin{equation*}
 \|\widehat m(z_*)-m_\psi(z_*)\|_\infty
 =O_p\!\left(h^\alpha+(Nh^d)^{-1/2}\right).
\end{equation*}
Balancing the terms gives $h\asymp N^{-1/(2\alpha+d)}$ and the stated rate.
\end{proof}

The exponent contains the local small-ball dimension, not necessarily the
ambient representation dimension.  If a learned representation has vanishing
local mass, $n_{\mathrm{eff}}$ exposes the failure: profile normalization
cannot manufacture relevant neighbors.

\subsection{Normalization, kernel approximation, and selection}

\begin{proof}[Proof of Equation~\ref{eq:normalization-stability}]
Let $E=\|e\|_1$, $\Delta=\sum_m\eta_m$, and
$\widetilde E=aE+\Delta$.  Since
$|\Delta|\leq\|\eta\|_1\leq\gamma aE$, we have
$\widetilde E\geq(1-\gamma)aE$.  Adding and subtracting
$ae/\widetilde E$ gives
\begin{align*}
 \left\|\frac{ae+\eta}{\widetilde E}-\frac eE\right\|_1
 &\leq\frac{\|\eta\|_1}{\widetilde E}
 +E\left|\frac a{\widetilde E}-\frac1E\right|\\
 &=\frac{\|\eta\|_1+|\Delta|}{\widetilde E}
 \leq\frac{2\gamma}{1-\gamma}.
\end{align*}
Setting $\eta=0$ proves invariance to every positive common scale.
\end{proof}

\begin{proposition}[Retrieval-kernel approximation]
\label{prop:kernel-approximation}
Suppose a realized retrieval kernel $\widetilde\kappa$ satisfies
$\sup_{u,v}|\widetilde\kappa(u,v)-\kappa(u,v)|\leq\epsilon$.  Then every
unnormalized atlas coordinate differs by at most $\epsilon$.  If the empirical
ideal kernel mass is $\widehat q(z)>\epsilon$, the normalized coordinate
differs by at most $2\epsilon/(\widehat q(z)-\epsilon)$.  An unnormalized
top-$k$ set is unchanged when its score margin exceeds $2\epsilon$.
\end{proposition}

\begin{proof}
Because $0\leq S_{im}\leq1$,
\begin{equation*}
 \left|\frac1N\sum_i
 [\widetilde\kappa(Z_i,z)-\kappa(Z_i,z)]S_{im}\right|\leq\epsilon.
\end{equation*}
The two empirical kernel masses also differ by at most $\epsilon$.  Writing
$A,\widetilde A$ for one ideal and approximate unnormalized coordinate and
$q,\widetilde q$ for their masses, $A\leq q$ and
$\widetilde q\geq q-\epsilon$, hence
\begin{equation*}
 \left|\frac{\widetilde A}{\widetilde q}-\frac Aq\right|
 \leq\frac{|\widetilde A-A|}{\widetilde q}
 +\frac{A|\widetilde q-q|}{q\widetilde q}
 \leq\frac{2\epsilon}{q-\epsilon}.
\end{equation*}
The margin statement follows because each compared coordinate moves by at
most $\epsilon$.  This result applies to Nystr\"om or any other approximate
realization; exact source and categorical kernels need no approximation step.
\end{proof}

\begin{proof}[Proof of Proposition~\ref{prop:selection}]
Let $S=S_k(v)$ and $\widehat S=S_k(\widehat v)$.  Optimality of
$\widehat S$ for $\widehat v$ gives
$\sum_{m\in S}\widehat v_m-\sum_{m\in\widehat S}\widehat v_m\leq0$.
Therefore
\begin{equation*}
 \sum_{m\in S}v_m-\sum_{m\in\widehat S}v_m
 \leq\sum_{m\in S}|v_m-\widehat v_m|
 +\sum_{m\in\widehat S}|v_m-\widehat v_m|
 \leq2kb.
\end{equation*}
If $m\in S$ and $j\notin S$, then
$\widehat v_m-\widehat v_j\geq v_m-v_j-2b
\geq\Delta_k(v)-2b>0$, proving exact recovery.
\end{proof}

\subsection{Bridge to the finite perturbation assay}

The selection results concern normalized first-order susceptibility.  The
following proposition quantifies when that target controls the
central-difference assay in Equation~\ref{eq:protocol-susceptibility}.

\begin{proposition}[Local perturbation bridge]
\label{prop:local-perturbation}
Fix $x$ and a selected parameter set $A$.  Suppose $F_\theta(x)$ is three
times differentiable on every segment $\theta+t\xi$, $|t|\leq\sigma$, used by
the assay, with $\|\nabla_\theta F_\theta(x)\|_2\leq B$ and third-derivative
operator norm at most $T$.  Let $\xi$ have independent Rademacher coordinates
on $A$, zero elsewhere, and satisfy $\|\xi\|_2\leq R_A$.  For
\begin{equation*}
 D_\sigma(\xi)
 =\frac{F_{\theta+\sigma\xi}(x)-F_{\theta-\sigma\xi}(x)}{2\sigma},
\end{equation*}
we have
\begin{equation}
 \left|
 \mathbb E_\xi[D_\sigma(\xi)^2]
 -\sum_{j\in A}(\partial_{\theta_j}F_\theta(x))^2
 \right|
 \leq\frac{TB\sigma^2R_A^4}{3}
 +\frac{T^2\sigma^4R_A^6}{36}.
 \label{eq:local-perturbation}
\end{equation}
For $D$ independent directions, if
$U_A=BR_A+T\sigma^2R_A^3/6$, then with probability at least $1-\delta$ the
empirical assay mean differs from its expectation by at most
\begin{equation}
 U_A^2\sqrt{\frac{\log(2/\delta)}{2D}}.
 \label{eq:finite-directions}
\end{equation}
\end{proposition}

\begin{proof}
The symmetric third-order Taylor formula gives
$D_\sigma(\xi)=\nabla F_\theta(x)^\top\xi+e(\xi)$ with
$|e(\xi)|\leq T\sigma^2\|\xi\|_2^3/6$.  Put
$a=\nabla F_\theta(x)^\top\xi$.  Since $|a|\leq BR_A$,
\begin{equation*}
 |(a+e)^2-a^2|
 \leq2|a||e|+e^2
 \leq\frac{TB\sigma^2R_A^4}{3}
 +\frac{T^2\sigma^4R_A^6}{36}.
\end{equation*}
Independence and zero mean of the Rademacher coordinates give
$\mathbb E_\xi[a^2]=\sum_{j\in A}(\partial_{\theta_j}F_\theta(x))^2$.
Also $|D_\sigma(\xi)|\leq U_A$, so Hoeffding's inequality for the $D$ bounded
squared derivatives proves Equation~\ref{eq:finite-directions}.
\end{proof}

The discrepancy between first-order and exact susceptibility contains a local
nonlinearity term of order $\sigma^2$ and a finite-direction term of order
$D^{-1/2}$.  The result does not assert that their constants are small for
every network or intervention.

\subsection{Scope of the statistical claims}

The finite-source theorem specializes standard kernel-regression reasoning to
the block-NTK sensitivity target.  Its purpose is to identify what controls
atlas fidelity: source coverage, stability of the NTK-derived target in the
chosen query space, and sensitivity variation that the representation does
not retain.  It is not a claim of a new generic kernel-regression rate.

Recovery of a realized query profile requires a small representation residual;
recovery of the conditional profile requires local mass, target stability,
and a suitable bandwidth.  Exact top-$k$ recovery additionally requires a
margin, and finite perturbations require the smooth local bridge above.

The independence assumption is appropriate for randomly sampled source
atlases.  The forecasting sources are deterministic, overlapping windows of
one series; a time-series extension would require mixing or other dependence
assumptions and is not claimed here.  Hyperparameters chosen on separate
development data satisfy the fixed-retrieval-rule setup, whereas unrestricted
selection on evaluation queries would require additional analysis.

\section{Additional Sensitivity Circuit Evidence}
\label{app:application2}

\subsection{Numerical results at the largest primary budget}

Table~\ref{tab:application2-operating-point} gives the exact values underlying
the 0.1\% endpoints in Figure~\ref{fig:application2-primary}. Reporting both
columns matters because a method can capture substantial gradient energy yet
fail to turn that energy into a comparable finite-perturbation effect.

\begin{table}[H]
  \centering
  \small
\begin{tabular}{llrrr}
\toprule
Setting & Method & Gap recovered (\%) & Gradient energy (\%) & $n$ \\
\midrule
Vision & Semantic KPSA & 34.7 [23.7, 47.1] & 71.8 [68.9, 74.7] & 169 \\
Vision & Categorical KPSA & 39.4 [28.8, 50.9] & 75.8 [73.1, 78.4] & 169 \\
Vision & KPSA-1NN & 28.1 [16.4, 40.5] & 69.5 [66.2, 72.6] & 169 \\
Vision & Constant sensitivity & 0.0 [0.0, 0.0] & 46.8 [45.1, 48.4] & 169 \\
Vision & Direct gradient & 100.0 [100.0, 100.0] & 100.0 [100.0, 100.0] & 169 \\
\addlinespace[2pt]
ETTh1 & Semantic KPSA & 46.5 [34.7, 56.2] & 72.9 [69.3, 76.2] & 128 \\
ETTh1 & Categorical KPSA & 2.2 [-4.4, 8.3] & 55.4 [50.6, 60.2] & 128 \\
ETTh1 & KPSA-1NN & 52.4 [36.0, 65.2] & 69.2 [64.7, 73.5] & 128 \\
ETTh1 & Constant sensitivity & 0.0 [0.0, 0.0] & 55.4 [50.9, 59.8] & 128 \\
ETTh1 & Direct gradient & 100.0 [100.0, 100.0] & 100.0 [100.0, 100.0] & 128 \\
\addlinespace[2pt]
Weather & Semantic KPSA & 50.2 [34.0, 66.4] & 74.6 [70.6, 78.2] & 128 \\
Weather & Categorical KPSA & 5.6 [-2.4, 14.0] & 56.4 [52.2, 60.6] & 128 \\
Weather & KPSA-1NN & 46.7 [31.7, 59.1] & 72.0 [68.0, 76.0] & 128 \\
Weather & Constant sensitivity & 0.0 [0.0, 0.0] & 56.3 [52.0, 60.7] & 128 \\
Weather & Direct gradient & 100.0 [100.0, 100.0] & 100.0 [100.0, 100.0] & 128 \\
\addlinespace[2pt]
Protein & Semantic KPSA & 33.2 [17.5, 55.5] & 42.8 [29.7, 57.0] & 48 \\
Protein & Categorical KPSA & 5.4 [-0.2, 13.0] & 12.6 [3.9, 25.0] & 48 \\
Protein & KPSA-1NN & 30.8 [16.5, 50.0] & 39.0 [25.7, 53.5] & 48 \\
Protein & Constant sensitivity & 0.0 [0.0, 0.0] & 7.2 [3.2, 12.9] & 48 \\
Protein & Direct gradient & 100.0 [100.0, 100.0] & 100.0 [100.0, 100.0] & 48 \\
\bottomrule
\end{tabular}
  \caption{Sensitivity results at a 0.1\% eligible-parameter budget. Entries are query means with paired 95\% bootstrap intervals. Gap recovery uses constant sensitivity and the direct query gradient as the 0\% and 100\% anchors. Gradient energy is measured relative to the direct query gradient.}
  \label{tab:application2-operating-point}
\end{table}

\FloatBarrier

\subsection{Cold-query and qualitative evidence}

\begin{figure}[H]
  \centering
  \includegraphics[width=.80\linewidth]{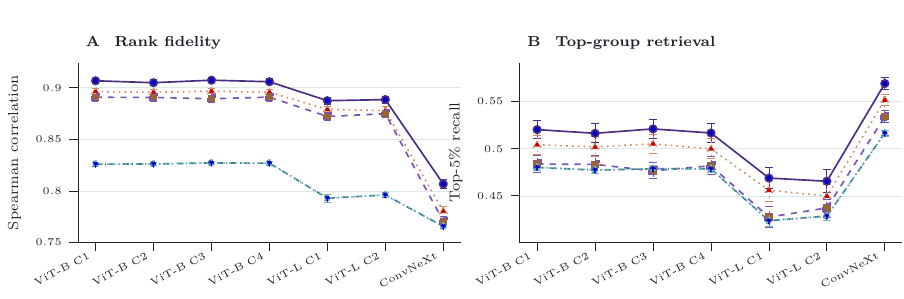}
  \caption{\textbf{Cold retrieval of held-out sensitivity profiles across architectures and confirmation splits.} Local semantic KPSA, KPSA-1NN, categorical KPSA, constant sensitivity, and a random-projection control are compared by rank-based profile metrics on ViT-B/16, ViT-L/16, and ConvNeXtV2-T.}
  \label{fig:cold-retrieval}
\end{figure}

\FloatBarrier

\begin{figure}[H]
  \centering
  \includegraphics[width=.84\linewidth]{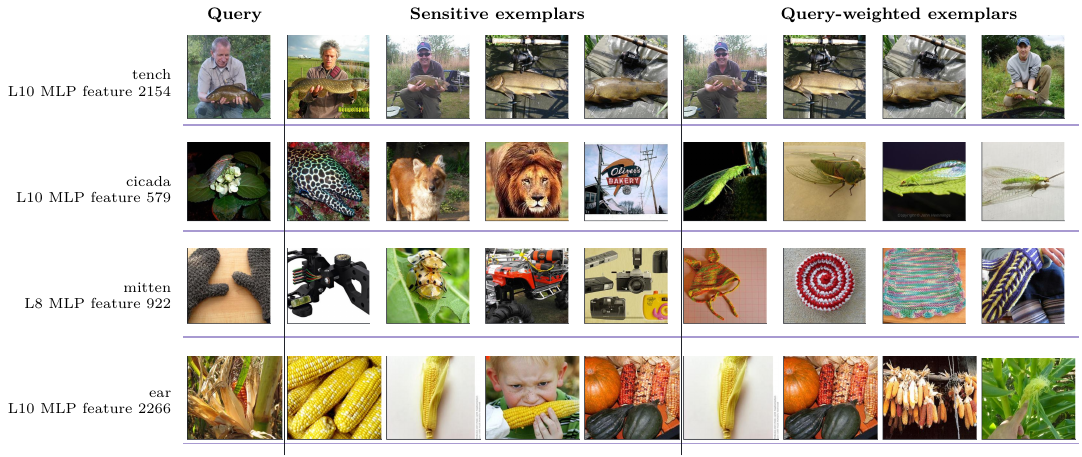}
  \caption{\textbf{Qualitative exemplars of sensitivity circuits.} ImageNet queries are paired with unconditioned sensitive source examples, query-weighted semantic neighbors, and the associated selected MLP feature groups. The examples illustrate the local visual neighborhoods summarized quantitatively in Figure~\ref{fig:application2-circuits} and are not standalone causal evidence.}
  \label{fig:circuit-exemplars}
\end{figure}

\FloatBarrier

\section{Additional Neuron-Localization Evidence}
\label{app:application1}

\subsection{Numerical results and measured cost}

\begin{table}[H]
  \centering
  \footnotesize
  \setlength{\tabcolsep}{4pt}
\begin{tabular}{llrr}
\toprule
Setting & Method & Gap recovered (\%) & $n$ \\
\midrule
Vision & Semantic KPSA & 66.5 [51.4, 84.1] & 76 \\
Vision & Categorical KPSA & 76.5 [62.8, 92.5] & 76 \\
Vision & KPSA-1NN & 74.8 [61.4, 90.1] & 76 \\
Vision & Constant sensitivity & 0.0 [0.0, 0.0] & 76 \\
Vision & Activation attribution & 100.0 [100.0, 100.0] & 76 \\
\addlinespace[2pt]
Protein & Semantic KPSA & 68.3 [46.3, 98.5] & 128 \\
Protein & Categorical KPSA & 4.4 [-0.7, 11.1] & 128 \\
Protein & KPSA-1NN & 56.6 [35.1, 86.5] & 128 \\
Protein & Constant sensitivity & 0.0 [0.0, 0.0] & 128 \\
Protein & Activation attribution & 100.0 [100.0, 100.0] & 128 \\
\bottomrule
\end{tabular}
  \caption{Neuron-localization results at a 0.1\% eligible-parameter budget. Entries are query means with paired 95\% bootstrap intervals. Constant sensitivity and action-matched activation attribution define the 0\% and 100\% anchors.}
  \label{tab:application1-operating-point}
\end{table}

\begin{table}[H]
  \centering
  \footnotesize
  \setlength{\tabcolsep}{3pt}
  \resizebox{\linewidth}{!}{
\begin{tabular}{lrrrr}
\toprule
Method and budget & Gap recovered (\%) & Images/s & Peak GiB & Atlas MiB \\
\midrule
Activation attribution & 100.0 & 551 & 4.24 & 0.0 \\
Semantic KPSA, 0.02\% & 64.1 & 2,041 & 0.41 & 112.5 \\
Semantic KPSA, 0.1\% & 80.9 & 2,041 & 0.41 & 112.5 \\
Semantic KPSA, 32-feature groups, 1/12 MLP parameters & 59.3 & 2,088 & 0.41 & 3.5 \\
\bottomrule
\end{tabular}}
  \caption{Measured vision accuracy--cost points from Figure~\ref{fig:application1-efficiency}. Throughput and peak memory use batch size 32 on one RTX 5090. Atlas storage is persistent float32 storage and excludes model weights. GiB and MiB use powers of two. The two feature-atlas rows share the same query cost because only the final selection budget changes.}
  \label{tab:application1-efficiency}
\end{table}

\FloatBarrier

\subsection{Ranking and resolution diagnostics}

\begin{figure}[H]
  \centering
  \includegraphics[width=.82\linewidth]{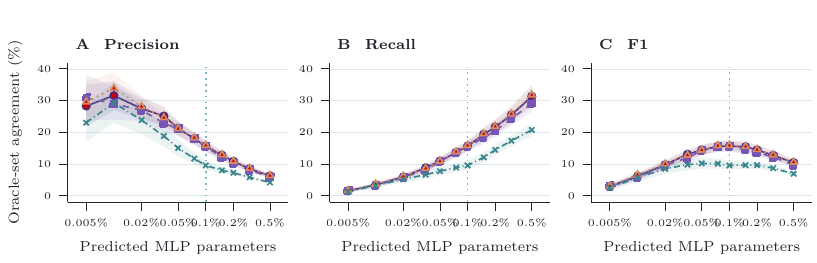}
  \caption{\textbf{Ranking precision--recall tradeoff against a fixed oracle set.} The oracle retains the top 0.1\% of ViT MLP parameters by positive gradient-times-activation score, while the predicted selection budget varies. Panels report parameter-mass-weighted precision, recall, and F1 for semantic KPSA, categorical KPSA, KPSA-1NN, and constant sensitivity. This fixed-oracle diagnostic is separate from the matched-budget overlap sweep used to select the 0.02\%, 0.05\%, and 0.1\% causal operating points.}
  \label{fig:budget-prf}
\end{figure}

\FloatBarrier

\begin{figure}[H]
  \centering
  \includegraphics[width=.82\linewidth]{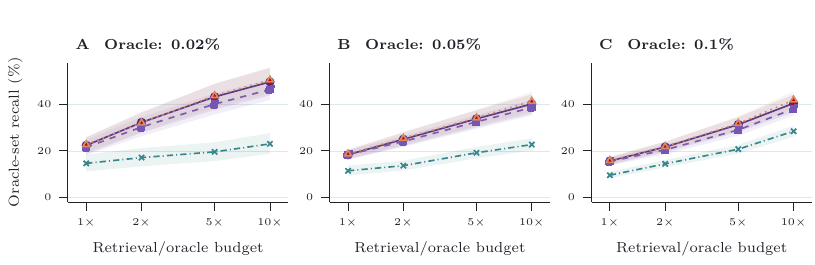}
  \caption{\textbf{Expanded retrieval budgets increase oracle-set recall for candidate generation.} Each panel fixes an activation-attribution oracle set at 0.02\%, 0.05\%, or 0.1\% of ViT MLP parameters, then retrieves $1\times$, $2\times$, $5\times$, or $10\times$ as many parameters. Semantic KPSA recalls 49.7\%, 40.2\%, and 40.5\% of the respective oracle sets at $10\times$, compared with 23.0\%, 22.7\%, and 28.4\% for constant sensitivity. Recall is central, whereas precision scales antiproportionally with the retrieval multiple.}
  \label{fig:budget-recall-multiples}
\end{figure}

\begin{figure}[H]
  \centering
  \includegraphics[width=.82\linewidth]{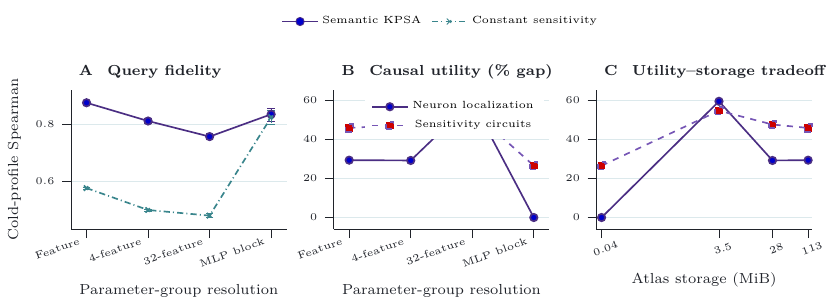}
  \caption{\textbf{Parameter-group resolution controls fidelity, causal utility, and storage.} Feature, 4-feature, 32-feature, and whole-MLP-block partitions are evaluated on cold-query fidelity and both causal applications. The final panel places causal utility against persistent atlas size. Intervals are 95\% confidence intervals.}
  \label{fig:parameter-resolution}
\end{figure}

\FloatBarrier

\end{document}